\documentclass{article}
\usepackage{geometry}
\usepackage[T1]{fontenc}
\usepackage[english]{babel}
\usepackage{mathtools,mathrsfs}
\mathtoolsset{showonlyrefs}
\usepackage{amsmath,amssymb,amsthm}
\usepackage{algorithm}
\usepackage[noend]{algpseudocode}
\usepackage{listings, enumitem}
\usepackage[nohints]{minitoc}
\usepackage{graphicx}
\usepackage{subcaption}
\usepackage{dsfont}
\usepackage{tikz}
\usepackage[nomessages]{fp}
\usepackage{pgfplots}
\pgfplotsset{compat=1.11}
\pgfplotsset{
	colormap={parula}{
	rgb255=(53,42,135)
	rgb255=(15,92,221)
	rgb255=(18,125,216)
	rgb255=(7,156,207)
	rgb255=(21,177,180)
	rgb255=(89,189,140)
	rgb255=(165,190,107)
	rgb255=(225,185,82)
	rgb255=(252,206,46)
	rgb255=(249,251,14)}
}
\usepackage{hyperref}

\DeclareMathOperator{\R}{{\mathbb{R}}}

\DeclareMathOperator*{\argmax}{arg\,max} 

\DeclareMathOperator{\diag}{diag}

\newcommand{\TV}[1]{\mathrm{TV}^{(#1)}}

\newcommand{\bigO}{\mathcal{O}}
\newcommand{\N}{\mathbb{N}}

\DeclareMathOperator{\Lip}{Lip}

\newcommand{\BV}[1]{\mathrm{BV}^{(#1)}}

\newcommand{\linlay}{A}

\newtheorem{theorem}{Theorem}[section]
\newtheorem{lemma}[theorem]{Lemma}
\newtheorem{remark}[theorem]{Remark}
\newtheorem{definition}[theorem]{Definition}

\newtheorem{corollary}[theorem]{Corollary}
\newtheorem{proposition}[theorem]{Proposition}

\title{Approximation of Lipschitz Functions \\using Deep Spline Neural Networks}
\date{\today}
\author{Sebastian Neumayer\footnotemark[1]\,\,\,\footnotemark[2] \and Alexis Goujon\footnotemark[1]\,\,\,\footnotemark[2] \and Pakshal Bohra\footnotemark[1] \and Michael Unser\footnotemark[1]}

\begin{document}

\maketitle
\renewcommand*{\thefootnote}{\fnsymbol{footnote}}
\footnotetext[1]{Biomedical Imaging Group, \'Ecole Polytechnique F\'ed\'erale de Lausanne (EPFL),
	Station 17, CH-1015 Lausanne, {\text \{forename.name\}@epfl.ch}. }
\footnotetext[2]{The first two authors contributed equally to this work.}
\renewcommand*{\thefootnote}{\arabic{footnote}}

\begin{abstract}
    Lipschitz-constrained neural networks have many applications in machine learning.
    Since designing and training expressive Lipschitz-constrained networks is very challenging, there is a need for improved methods and a better theoretical understanding.
    Unfortunately, it turns out that ReLU networks have provable disadvantages in this setting.
    Hence, we propose to use learnable spline activation functions with at least 3 linear regions instead.
    We prove that this choice is optimal among all component-wise $1$-Lipschitz activation functions in the sense that no other weight constrained architecture can approximate a larger class of functions.
    Additionally, this choice is at least as expressive as the recently introduced non component-wise Groupsort activation function for spectral-norm-constrained weights.
    Previously published numerical results support our theoretical findings.
\end{abstract}


\section{Introduction}
Throughout the past years, Lipschitz-constrained neural networks (NNs) have proven to be useful in several areas of machine learning, e.g., for provably convergent Plug-and-Play algorithms \cite{HNS2021, MMHC2017,ryu2019plug,SVW2016,TRPW20,VBW13}, to obtain robustness guarantees \cite{HN20,Pauli2022,TSS2018} or in Wasserstein GANs \cite{pmlr-v70-arjovsky17a,GAAD2017}.
Unfortunately, designing and training Lipschitz-constrained NNs is difficult, as naive upper bounds on the Lipschitz constant of multi-layer models are often too pessimistic.
More advanced estimators for this NP-hard problem are based on semi-definite programming \cite{FRH2019,LRC2020}, adversarial training \cite{BRRS2021,RKH2020} or deriving sharper estimates for the composition of layers \cite{VS2018}.
Unfortunately, these methods are either computationally expensive, or do not provide a proper upper bound.

Another possible approach for tackling the problem is to improve the model architecture so that the naive bounds become sharper.
A general overview of NN architectures and in particular Lipschitz-constrained ones can be found in \cite{Calin2020}.
The most common approach towards Lipschitz-constrained architectures is to bound the norm of each linear layer by some constant, e.g., in form of the spectral or other $p$-norms \cite{GFPC18,MKKY2018,SGL2019}.
Other approaches go even further and enforce orthogonality of the weight matrices \cite{HHNPSS2019,HNS2021,huang2018orthogonal}.
In combination with $1$-Lipschitz activations, this results in architectures with a Lipschitz constant bounded by the product of the norms of the weights.
However, this estimate is in general quite pessimistic, especially for deep models. Consequently, this additional structural constraint often leads to vanishing gradients \cite{LHAL2019} and seriously reduced expressivity of the model.
Remarkably, the commonly used ReLU aggravates the situation even more.
For instance, it is shown in \cite{HCC2018} that ReLU NNs with $\infty$-norm weight constraints have second order total variation bounded independently of the depth.
Further, it is proven in \cite{ALG19} that under spectral norm constraints, any scalar-valued ReLU NN $\Phi$ with $\Vert \nabla \Phi \Vert_2 = 1$ a.e.\ is necessarily linear.
To circumvent the described issues, several new activation functions have been proposed recently, e.g., Groupsort \cite{ALG19}, the related Householder activation functions \cite{SSF2021} or the Weibull activation function \cite{pmlr-v139-zhao21e}.
Note that, contrary to the ReLU, all these activation functions are multivariate. 
Analyzing the expressivity of the resulting NNs and determining their applicability in practice is an active area of research.

Given a NN with $1$-Lipschitz layers, it is by no means clear which class of functions it can approximate.
Ideally, given a compact set $D \subset \R^d$ equipped with the $p$-norm, it should be possible to approximate all scalar-valued $1$-Lipschitz functions, which are denoted by $\Lip_{1,p}(D)$.
The first result in this direction is provided in \cite{ALG19}, where the authors show that using the Groupsort activation function and $\infty$-norm-constrained weights indeed allows universal approximation of $\Lip_{1,p}(D)$.
The behavior of such NNs was then further investigated in \cite{CHC2019,TSB2021}.
Unfortunately, the employed proof strategies cannot be generalized to other norms.
So far, not even partial results are known for this very challenging problem, and comparing the approximation power of different architectures is an important first step.
From a practical perspective, Groupsort NNs have yielded promising results, which compare favorably against ReLU NNs with similar architectures \cite{ALG19}.
A classic benchmark example is approximating the absolute value function, for which exact representation with ReLU is impossible.

Most substantial results in this area so far rely on multivariate activation functions. 
Although the ReLU activation function is indeed too limiting, the class of component-wise activation functions should not be written off too early.
Following this idea, we analyze deep spline NNs, whose activation functions are learnable linear splines \cite{AU2019,BCGA2020, unserRepresenterTheoremDeep2019}.
Since bounds on the Lipschitz constant of compositions are usually too pessimistic, the rationale is to increase the expressivity of the activation function while still being able to efficiently control its Lipschitz constant.
As reported in \cite{bohra2021learning}, Lipschitz-constrained deep spline NNs perform well in practice.
In this work, we shed light on the theoretical benefits of these NNs over ReLU-like NNs.
In particular, we prove that among weight-constrained NNs with component-wise $1$-Lipschitz activation functions, splines with 3 linear regions suffice to approximate the largest possible set of functions. Moreover, for the spectral norm constraint, which is commonly used in practice, we show that deep spline NNs are at least as expressive as Groupsort NNs. Due to these theoretical findings, we expect them to be beneficial in applications that involve Lipschitz constraints.


\paragraph{Outline and contributions}
In Section~\ref{sec:PreLiminaries}, we revisit $1$-Lipschitz continuous piecewise linear (CPWL) functions and $1$-Lipschitz  NNs.
Here, we show that they can approximate any function in $\Lip_{1,p}(D)$, where $D \subset \R^d$ is compact.
Since constructing $1$-Lipschitz NNs is non-trivial, we briefly discuss two architectures for this task, namely deep spline and Groupsort NNs.
Then, in Section~\ref{sec:LimitReLU}, we formulate extensions of results on the limitations of weight-constrained NNs with ReLU activation functions.
More precisely, we show that ReLU-like NNs cannot represent certain simple functions for any $p$-norm weight constraint.
Based on a second-order total variation argument, we further show that they cannot be universal approximators for $\infty$-norm weight constraints.
Next, in Section~\ref{sec:ApproxLipNet}, we study the approximation properties of deep spline NNs.
Here, we prove our main result, i.e., that deep spline NNs with 3 linear regions achieve the maximum expressivity among NNs with component-wise activation functions.
Further, we discuss the relation between deep spline and Groupsort NNs.
Finally, we draw conclusions in Section~\ref{sec:Conclusions}.

\section{Lipschitz-Constrained Neural Networks}\label{sec:PreLiminaries}
In this paper, we investigate general NN architectures consisting of $K \in {\mathbb N}$ layers with widths $n_{1}, \ldots , n_{K}$ that are given by mappings 
$\Phi = \Phi(\cdot\,; u)\colon {\mathbb R}^{d} \to {\mathbb R}^{n_{K}}$ of the form 
\begin{equation}\label{eq:NN_archi}
\Phi\left(x;u\right) \coloneqq A_K \circ\sigma_{K-1,\alpha_{K-1}} \circ A_{K-1}\circ\sigma_{K-2,\alpha_{K-2}} \circ \cdots \circ \sigma_{1,\alpha_1} \circ A_1(x).
\end{equation}
Here, the affine functions $A_{k}\colon {\mathbb R}^{n_{k-1}} \to {\mathbb R}^{n_{k}}$ are given by
\begin{equation} \label{act1} 
A_{k}(x) \coloneqq W_{k} x + b_{k}, \qquad  k =1,\ldots, K, 
\end{equation}
with weight matrices $W_{k} \in {\mathbb R}^{n_{k}, n_{k-1}}$, $n_{0}=d$ and bias vectors $b_{k} \in {\mathbb R}^{n_{k}}$. Further, the model includes parametrized nonlinear activation functions $\sigma_{k,\alpha}\colon\mathbb{R}^{n_k} \rightarrow \mathbb{R}^{n_k}$ with corresponding parameters $\alpha_k$, $k=1,\ldots, K-1$.
For the case of component-wise activation functions, we have $\sigma_{k,\alpha} (x) = (\sigma_{k,\alpha,j}(x_{j}))_{j=1}^{n_k}$.
The complete parameter set of the NN is denoted by  \smash{$u\coloneqq\left(W_k,b_k,\alpha_k\right)_{k=1}^{K}$}.
For an illustration see Figure~\ref{fig:neuralnet}.
We sometimes drop the index $k$ in the activation function $\sigma_{k,\alpha}$ and the dependence on the parameter $u$ in $\Phi$ to simplify the notation.
Recall that the architecture ~\eqref{eq:NN_archi} results in a CPWL function whenever the activation functions themselves are CPWL functions such as the ReLU.
Next, we investigate the approximation properties of this architecture under Lipschitz constraints on $\Phi(\cdot,u)$.

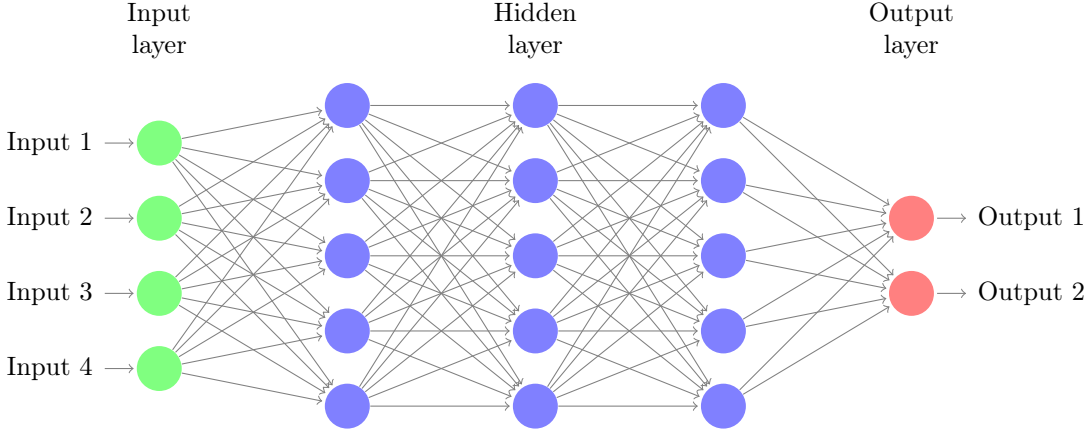
\begin{figure}[t]
\def\layersep{2.5cm}
\centering
\begin{tikzpicture}[shorten >=1pt,->,draw=black!50, node distance=\layersep]
    \tikzstyle{every pin edge}=[<-,shorten <=1pt]
    \tikzstyle{neuron}=[circle,fill=black!25,minimum size=17pt,inner sep=0pt]
    \tikzstyle{input neuron}=[neuron, fill=green!50];
    \tikzstyle{output neuron}=[neuron, fill=red!50];
    \tikzstyle{hidden neuron}=[neuron, fill=blue!50];
    \tikzstyle{annot} = [text width=4em, text centered]

    \foreach \name / \y in {1,...,4}
        \node[input neuron, pin=left:Input \y] (I-\name) at (0,-\y) {};

    \foreach \name / \y in {1,...,5}
        \path[yshift=0.5cm]
            node[hidden neuron] (Ha-\name) at (\layersep,-\y cm) {};

    \foreach \name / \y in {1,...,5}
        \path[yshift=0.5cm]
            node[hidden neuron] (Hb-\name) at (2*\layersep,-\y cm) {};
            
    \foreach \name / \y in {1,...,5}
        \path[yshift=0.5cm]
            node[hidden neuron] (Hc-\name) at (3*\layersep,-\y cm) {};        

    \node[output neuron,pin={[pin edge={->}]right:Output 1}, right of=Hc-3] (O1) at (3*\layersep,-2 cm) {};
    \node[output neuron,pin={[pin edge={->}]right:Output 2}, right of=Hc-3]  (O2) at (3*\layersep,-3 cm) {};

    \foreach \source in {1,...,4}
        \foreach \dest in {1,...,5}
            \path (I-\source) edge (Ha-\dest);
            
    \foreach \source in {1,...,5}
        \foreach \dest in {1,...,5}
            \path (Ha-\source) edge (Hb-\dest);
            
    \foreach \source in {1,...,5}
        \foreach \dest in {1,...,5}
            \path (Hb-\source) edge (Hc-\dest);

    \foreach \source in {1,...,5}
        \path (Hc-\source) edge (O1);
        
    \foreach \source in {1,...,5}
        \path (Hc-\source) edge (O2);

    \node[annot,above of=Hb-1, node distance=1cm] (hm) {Hidden layer};
    \node[annot,above of=Ha-1, node distance=1cm] (hl) {};
    \node[annot,above of=Hc-1, node distance=1cm] (hr) {};
    \node[annot,left of=hl] {Input layer};
    \node[annot,right of=hr] {Output layer};
\end{tikzpicture}
\caption{Model of a feed forward NN with three hidden layers, i.e., $d=4$, $K=4$, $n_1=n_2=n_3=5, n_4=2$.}
\label{fig:neuralnet}
\end{figure}

\subsection{Universality of 1-Lipschitz ReLU Networks}\label{sec:Universality}
First, we briefly revisit the approximation of Lipschitz function by CPWL functions, for which we give a precise definition with related notations below.
\begin{definition}
  \label{df:cpwl}
  A continuous function $f \colon \R^d \rightarrow \R^n$ is called continuous and piecewise linear (CPWL) if there exist a set $\{f^m\colon m=1,\ldots,M\}$ of affine functions, also called affine pieces, and closed subsets $(\Omega_m)_{m=1}^M$ of $\R^d$ with nonempty and pairwise disjoint interiors, also called projection regions~\cite{Tarela1990}, such that $\cup_{m=1}^M\Omega_m = \R^d$ and $f_{|\Omega_m}=f^m_{|\Omega_m}$.
\end{definition}
Assume that we are given a collection of tuples $(x_i,y_i) \in \R^d \times \R$, $i=1,\ldots,N$, which can be interpreted as samples from a function $f \colon \R^d \to \R$.
Let
\[L_{x,y}^p \coloneqq \max_{i,j} \frac{\vert y_i - y_j \vert}{\Vert x_i -x_j \Vert_p}\]
denote the Lipschitz constant associated with these points.
Then, a first natural question is whether it is always possible to find an interpolating CPWL function $g$ with $p$-norm Lipschitz constant $\Lip_p (g)= L_{x,y}^p$.
\begin{proposition}\label{prop:GeneralApprox}
    For the tuples $(x_i,y_i) \in \R^d \times \R$, $i=1,\ldots,N$, there exists a CPWL function $g$ with $\Lip_p (g) = L_{x,y}^p$, $p \in [1,\infty]$, such that $g(x_i) = y_i$ for all $i=1,\ldots,N$.
\end{proposition}
Since we are not aware of a proof for general $p$, we provide one below.
\begin{proof}
    Let $q$ be such that $1/p+1/q = 1$.
    For $p < \infty$, define $u_{ij}\in\mathbb{R}^d$ as the vector given by
    \[
    (u_{ij})_k = \mathrm{sign}((x_i - x_j)_k)|(x_i - x_j)_k|^{p/q}.
    \]
    If $p=+\infty$, we choose $k_0$ with $\Vert x_i - x_j\Vert_{\infty}= \vert(x_i - x_j)_{k_0}\vert$, and define $(u_{ij})_{k_0} = \mathrm{sign}(x_i - x_j)_{k_0}$ with all other components of $u_{ij}$ set to 0.
    This saturates the Hölder's inequality with
    \[
    \langle u_{ij},x_j-x_i \rangle = \sum_{k=1}^d \vert (u_{ij})_k (x_j-x_i)_k  \vert = \Vert u_{ij} \Vert_q \Vert x_j - x_i \Vert_p ,
    \]
    where we used that $u_{ij}$ and $(x_j-x_i)$ have components with the same sign. Define the linear function
    \[g_{i,j}(x) = y_i + \frac{y_j-y_i}{\Vert x_j - x_i \Vert_p \Vert u_{i,j}\Vert_q}\langle u_{ij},x-x_i \rangle,\]
    which is such that $g_{i,j}(x_i)=y_i$ and $\Lip_p(g_{i,j}) = \vert y_j-y_i\vert/\Vert x_j - x_i \Vert_p$, as $\sup_{\Vert x\Vert_p\leq 1}\langle u_{ij},x\rangle = \Vert u_{ij} \Vert_q$.
    Next, set $g_i(x) = \max_{j} g_{i,j}(x)$, for which it holds $g_{i}(x_i)=y_i$ and $\Lip_p(g_{i}) = \max_j \vert y_j-y_i \vert/\Vert x_j - x_i \Vert_p$.
    Then, we define $g(x)= \min_{i} g_i(x)$ and directly obtain $g(x_j)\leq y_j$ for any $j=1,\ldots,N$.
    However, we also have
    \[g_i(x_j) \geq g_{i,j}(x_j) = y_i + y_j - y_i = y_j,\]
    which then implies $g(x_j) = y_j$ for any $j=1,\ldots,N$.
    Further, we directly get that $\Lip_p(g) = L_{x,y}^p$. Finally, by recalling that the maximum and the minimum of any number of CPWL functions is CPWL as well, we conclude that $g$ is CPWL and the claim follows.
\end{proof}
\begin{remark}
Here, we already see that the construction is more involved than in the 1D case.
In general, an arbitrary triangulation of the data points leads to an interpolation with a non-optimal Lipschitz constant, see also Figure~\ref{fig:triangulationFailure}.
\end{remark}

\begin{figure}[t]
  \centering
  \centerline{\includegraphics[width=150mm]{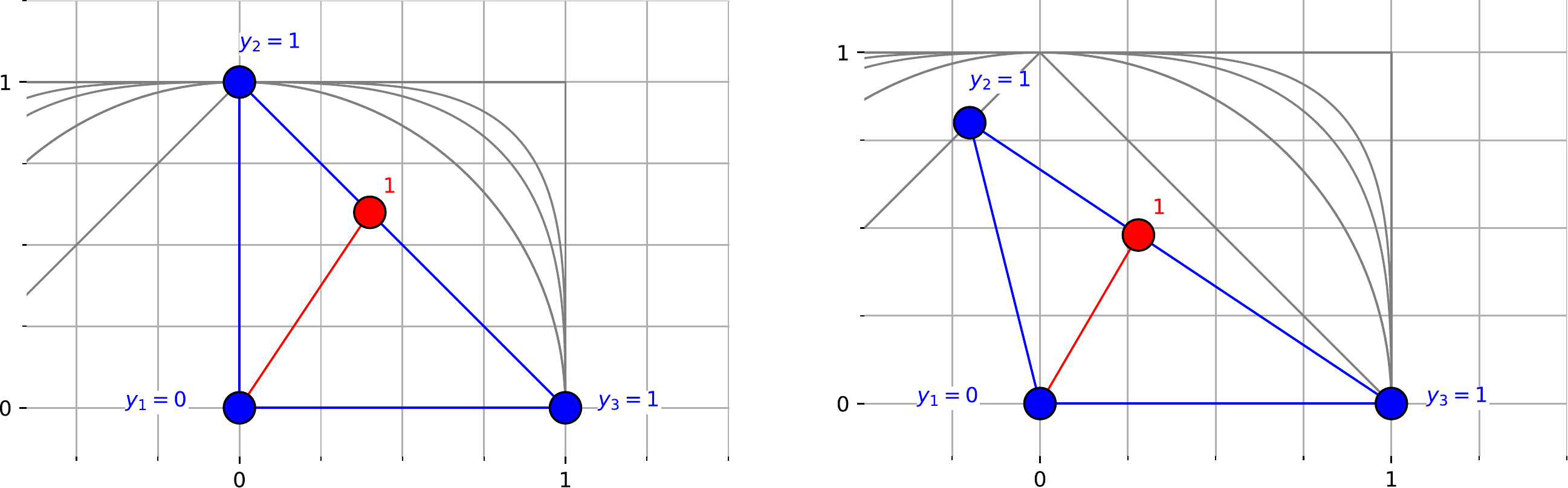}}
  \caption{Interpolating data points based on a triangulation might produce a CPWL function whose Lipschitz constant exceeds $L_{x,y}^p$. Here, the affine function going through the three blue data points $(x_k,y_k)\in (\mathbb{R}^2,\mathbb{R})$ has a $p$-norm Lipschitz constant greater than 1, where $p\in (1,+\infty]$ for the figure on the left, and $p=1$ for the one on the right. The gray lines depict the $\ell_p$ unit balls for $p\in\{1,2,3,4,+\infty\}$.}
  \label{fig:triangulationFailure}
\end{figure}
Since the maximum and minimum of finitely many affine functions can be represented by ReLU NNs, the same holds for the constructed CPWL function in Proposition~\ref{prop:GeneralApprox}.
This directly leads us to the following well-known corollary.
\begin{corollary}\label{cor:DenseLip}
Let $D \subset \R^d$ be compact and $p \in [1,\infty]$.
Then, the ReLU NNs $\Phi\colon D \to \R$ with $\Lip_p(\Phi) \leq 1$ are dense in $\Lip_{1,p}(D)$. 
\end{corollary}
Since computing the Lipschitz constant of a NN is in principal NP-hard, Corollary \ref{cor:DenseLip} has limited practical relevance.
To circumvent this issue, algorithms providing tight estimates or special architectures with simple yet sharp bounds are necessary.
In this paper, we pursue the second direction, and introduce the necessary tools for building Lipschitz-constrained architectures in the remainder of this section.
As a natural follow-up question, we then investigate the universality of these restricted architectures in Section \ref{sec:ApproxLipNet}. 

\subsection{1-Lipschitz Neural Network Architectures}\label{sec:LipArch}
A first step towards Lipschitz-constrained NNs is to constrain the weights, for which several possibilities exist.
As we are aiming for 1-Lipschitz NNs, we always choose the constraints to be one, but remark that other values are possible as well.
If we further impose that all activation functions $\sigma_{k,\alpha}$ are $1$-Lipschitz, then the resulting NN is also $1$-Lipschitz.

\paragraph{Operator norm constraints}
The $p \rightarrow q$ operator norm is given for $W \in \R^{n,m}$ and $p,q \in [1,\infty]$ by
\[\Vert W \Vert_{p,q} = \max_{x \in \R^m, \Vert x \Vert_p = 1} \Vert W x \Vert_q,\]
and $\Vert \cdot \Vert_{p} \coloneqq \Vert \cdot \Vert_{p,p}$. Note that $\Vert \cdot \Vert_1$ and $\Vert \cdot \Vert_\infty$ correspond to the maximum $\ell_1$ norm of the columns and rows of $W$, respectively.
The norm $\Vert \cdot \Vert_2$, also known as spectral norm, corresponds to the largest singular value of $W$.
To obtain a non-expansive NN of the form \eqref{eq:NN_archi} in the $p$-norm sense, the weight matrices can be constrained as
\[\Vert W_k \Vert_p \leq 1, \quad k=1,\ldots,K,\]
which we will refer to as $p$-norm-constrained weights throughout the remainder of the paper.
For matrices $W \in \R^{1,n}$ it holds $\Vert W \Vert_p = \Vert W^T \Vert_q$ with $1/p + 1/q = 1$, i.e., if we interpret them as vectors, then we have to constrain the $q$-norm instead.
In case of scalar-valued NNs, we can also constrain the weights as $\Vert W_k \Vert_q \leq 1$, $k=2,\ldots,K$, and $\Vert W_1 \Vert_{p,q} \leq 1$, since all norms are identical in $\R$.
There exist several methods for enforcing such constraints in the training stage, see \cite{GFPC18,MKKY2018,SGL2019}.

\paragraph{Orthogonality constraints}
Instead of imposing constraints on $\Vert \cdot \Vert_2$, we can also require that either $W^T W = \text{Id}$ or $W W^T = \text{Id}$, depending on the shape of $W$.
This constraint corresponds to asking that either $W$ or $W^T$ lie in the so-called Stiefel manifold.
Compared to the spectral norm constraint, this enforces all singular values of $W$ to be one.
From a computational perspective, this approach is more challenging than the previous one but helps to mitigate the problem of vanishing gradients in deep NNs.
For more details, including possible implementations, we refer to \cite{HHNPSS2019,HNS2021,huang2018orthogonal}.

\begin{remark}
Many of the implementations for the above-mentioned schemes only enforce the $p$-norm constraint or orthogonality approximately.
For theoretical guarantees, it is however necessary to ensure that the constraint is exactly satisfied. In practice, this means that sufficient numerical accuracy or additional post-processing after training might be necessary.
\end{remark}

\subsection{Special Activation Functions}
As discussed in Section \ref{sec:LipArch}, we want to use a $1$-Lipschitz activation function in \eqref{eq:NN_archi}.
Here, the quest for optimal ones in the last decade leaves us with many choices.
However, the $1$-Lipschitz constraint is the game-changer, and the relevance of each activation function must be reassessed.
In Section \ref{sec:LimitReLU}, we provide results that explain why the ReLU activation function is actually not well suited for the Lipschitz-constrained setting.
Hence, we need to resort to other activation functions that lead to increased expressivity of the resulting NN.
Note that there is a fundamental difference between component-wise and general multivariate activation functions. Finding a good trade-off in terms of representational power and computational complexity is therefore necessary.
In the following, we briefly discuss one family of activation functions for each case.
Both discussed activation functions have proven to be well-suited in the constrained setting with promising experiments.
In the remainder of the paper, we further explore their usability in the norm-constrained case and investigate the relations between the two approaches.

\paragraph{Deep spline neural networks}
A deep spline NN \cite{AU2019,BCGA2020, unserRepresenterTheoremDeep2019} uses learnable component-wise linear spline activation functions. For an illustration, see Figure~\ref{fig:linearspline}.
The rationale for this family is to generalize the popular and computationally efficient ReLU to obtain a higher expressivity, while still being able to exactly control the Lipschitz constant of each activation function.
Any such activation function is fully characterized by its linear regions and the corresponding values at the boundaries.
In the unconstrained setting, any linear spline can be implemented by means of a scalar one hidden layer ReLU NN as
\begin{equation}
    x\mapsto \sum_{m=1}^M u_{m}\mathrm{ReLU}(v_{m}x + b_{m}),
\end{equation}
where $u_{m},v_{m},b_{m} \in\mathbb{R}$ and $M\in\mathbb{N}$.
This parametrization, however, lacks expressivity under $p$-norm constraints on the weights, as it is not able to produce linear spline with second order total variation greater than 1, see Lemma~\ref{lm:netTV2} and Section~\ref{subsec:variations} for more details.
The use of deep spline NNs overcomes this limitation.
In practice, the linear spline activation functions have a fixed number of uniformly spaced breakpoints and are parametrized by the cardinal B-splines, also known has the hat basis functions \cite{BCGA2020}.
While the implementation requires additional learnable parameters---the number of breakpoints plus 2, per activation function---the evaluation complexity remains independent of the number of breakpoints. More interestingly, the Lipschitz constant of the activation function can be efficiently and precisely controlled in the learning stage \cite{bohra2021learning}.
Among weight-constrained NNs with component-wise activation functions, deep spline NNs achieve the optimal representational power.

\begin{figure}[t]
  \centering
  \centerline{\includegraphics[width=100mm]{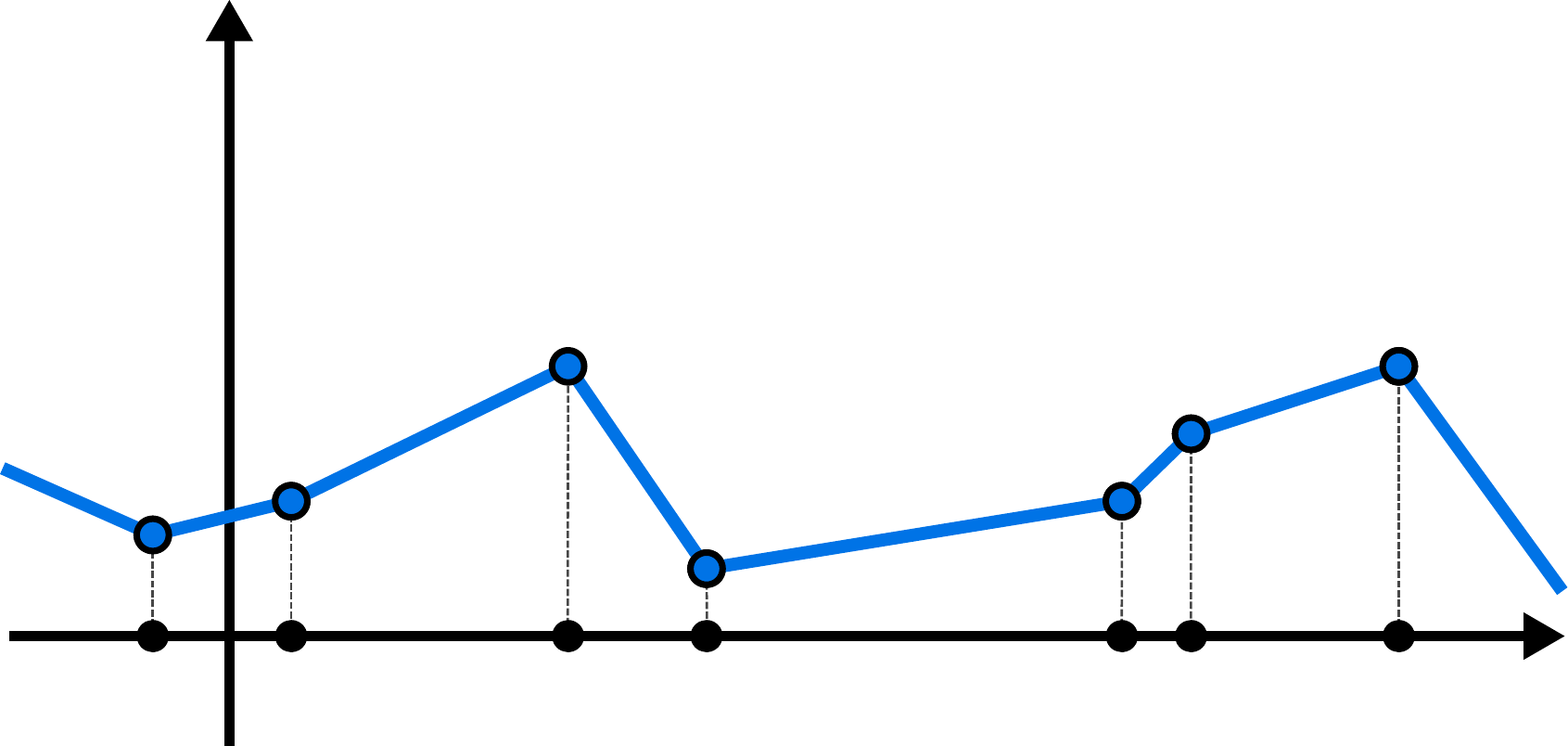}}
  \caption{A linear spline with 7 knots (also known as breakpoints) and 8 linear regions.}
  \label{fig:linearspline}
\end{figure}

\begin{lemma}
Let $(x_n,y_n)\in (\mathbb{R}^d,\mathbb{R}^p)$, $n=1,\ldots,N$, be training points and $\Phi$ a NN with $K$ layers, parameter set $u$, $p$-norm weight constraints and 1-Lipschitz activation functions.
Then, there exists a deep spline NN denoted by $\mathrm{DS}$ with the same architecture and activation functions replaced by a $1$-Lipschitz linear spline with no more than $(N-1)$ linear regions such that
\begin{equation}
    \Phi(x_n,u) = \mathrm{DS}(x_n,u) \text{ for } n=1,\ldots,N.
\end{equation}
\end{lemma}
\begin{proof}
On the data points $(x_n,y_n)_{n=1}^N$, the activation functions of $\Phi$ are evaluated for at most $N$ different values.
Hence, the result directly follows by interpolating between theses values using a linear spline, which yields 1-Lipschitz linear spline activation functions.
\end{proof}
This result is somehow still unsatisfying as the number of linear regions grows with the number of training points $N$.
Later, we show that linear splines activation functions with 3 linear regions are actually sufficient, which from a theoretical point of view amounts to 6 tunable parameters per activation function.

\paragraph{Groupsort}
The sort operation takes an input vector of length $n$ and simply outputs its components sorted in ascending order.
This operation has complexity $\bigO(n \log(n))$, which is slightly worse than the linear complexity of component-wise activation functions.
The Groupsort activation function \cite{ALG19} is a generalisation of this operation:
It splits the pre-activation into groups of prescribed length and performs the sort operation within each group.
This also makes the complexity linear again.
If the group size is two, then the activation function is also known as MaxMin or norm preserving orthogonal permutation linear unit \cite{CN2016}. 
Let us remark that any arbitrary Groupsort activation function can be written as composition of MaxMin activation function, i.e., larger group sizes do not increase the theoretical expressivity.
Although not obvious at first glance, the Groupsort activation function is actually a CPWL operation.
The rationale for this activation function is to perform a nonlinear and norm preserving operation, which mitigates the issue of vanishing gradients in deep constrained architectures.
More precisely, we have that the Jacobian of the Groupsort activation function is a.e.\ given by a permutation matrix, i.e. an orthogonal matrix.
Motivated by this observation, this approach was recently generalized towards Householder activation functions $\sigma_v \colon \R^d \to \R^d$ with $v \in \R^d$, $\Vert v\Vert_2=1$, given by
\[\sigma_v(z) = \begin{cases}
 z & \mbox{ if } v^T z > 0,\\
 (\text{Id} - 2vv^T)z & \mbox{ if } v^T z  \leq 0,
\end{cases}\]
see \cite{SSF2021}.
At the hyperplane separating the two cases, i.e., $v^T z = 0$ we have $(I - 2vv^T)z = z - 2(v^T z)v = z$.
Thus, $\sigma_v$ is continuous and, moreover, the Jacobian is either I or
$(I - 2vv^T)$, which are both square orthogonal matrices.
For practical purposes, the authors recommend to also use groups of size 2.
This construction can be iterated to obtain Householder activation functions of higher order with more linear regions.

\section{Limitations of Certain Architectures}\label{sec:LimitReLU}
In this section, we provide results that explain why using more involved activation functions than the ReLU is indeed necessary for weight-constrained NNs.
\subsection{Diminishing Jacobians}
Using component-wise and monotone activation functions is known to be detrimental for the expressivity of NNs with spectral-norm-constrained weights  \cite[Thm.~1]{ALG19}.
In the following, we generalize this result to NNs with $p$-norm-constrained weights and certain CPWL activation functions along with a more precise characterization, i.e., we also cover the case where $\Vert J \Phi \Vert_p$ is not $1$ a.e.

\begin{proposition}
\label{pr:limitationReLUNet}
    Let $p\in(1,+\infty]$, $I \subset \R$ a closed interval and $\sigma \colon \R \to \R$  be a component-wise CPWL activation function with $\sigma(x) = x + b$, $b\in \R$, for $x \in I$ and $\vert \sigma^\prime(x) \vert \leq c <1$ otherwise.
    Any NN $\Phi \colon \mathbb{R}^d\rightarrow\mathbb{R}$ of the form~\eqref{eq:NN_archi} with $p$-norm-constrained weights and activation function $\sigma$ generates a CPWL function that has at most one affine piece on which $\Vert J \Phi \Vert_p = 1$.
\end{proposition}
\begin{proof}
    We proceed via induction over the number of layers $K$ of $\Phi$.
    For $K=1$, the mapping is affine and the statement trivially holds.
    Now, assume that the result holds for some $K>1$.
    Let 
    \[\Phi_{K+1} = \linlay_{K+1}\circ \sigma\circ \linlay_{K}\circ\cdots\circ \sigma\circ \linlay_1,\]
    which we decompose as $\Phi_{K+1} = \Phi_{K} \circ h$ with $\Phi_{K} = \linlay_{K+1}\circ \sigma\circ \linlay_K\circ\cdots\circ \sigma\circ \linlay_2$ and  $h = \sigma\circ \linlay_1$.
    The induction assumption implies that $\Vert J \Phi_K \Vert_p < 1$ on all affine regions except possibly one.
    The corresponding affine function $g_K^1\colon \mathbb{R}^{n_1}\rightarrow\mathbb{R}$ with projection region $\Omega_K \subset\mathbb{R}^{n_1}$ takes the form $x\mapsto v^T x + b$, where $v\in\mathbb{R}^{n_1}$ is such that $\|v\|_q\leq 1$, $1/p+1/q = 1$, and $b\in\mathbb{R}$.
    Now, we define the set
    \[\Omega_{K+1} = \{x\in\R^d : (A_1(x))_i \in I \text{ for any } i \text{ s.t. } v_i \neq 0\} \cap h^{-1}(\Omega_{K}).\]
    By construction, $\Phi_{K+1}$ is affine on $\Omega_{K+1}$ and coincides with $\Phi_K\circ(\linlay_1 + b)$ on this set.
    Any other affine piece of $\Phi_{K+1}$ can be written in the form of $g_K^{i}\circ h^j$, where $g_K^{i}$ and $h^j$ are affine pieces of $\Phi_K$ and $h$, respectively.
    For this composition, either of the two following holds:
    \begin{itemize}
        \item[i)] It holds $g_K^{i}\neq g_K^1$, which results in $\Vert J (g_K^{i}\circ h^j)) \Vert_p < 1$ due to $\Vert J g_K^{i} \Vert_p < 1$.
        \item[ii)] It holds $g_K^{i} = g_K^1$, which due to the definition of $\Omega_{K+1}$ implies $J h^j = \diag(d) A_1$ for some $d \in \R^{n_1}$ with entries $\vert d_i \vert \leq 1$.
        Further, there exists $i^*$ such that $v_{i^*} \neq 0 $ and $\vert d_{i^*} \vert < 1$.
        Hence, the Jacobian of the affine piece is given by $\tilde{v}^T W_1$ with $\tilde{v} = \diag(d) v$.
        Since $p\neq 1$, we get that $q<+\infty$ and $\|\tilde{v}\|_q<\|v\|_q\leq 1$.
        Consequently, $\Vert J (g_K^{i}\circ h^j)) \Vert_p = \|\tilde{v}^T W_1\|_p\leq \|\tilde{v}\|_q\|W_1\|_p<1$.
    \end{itemize}
This concludes the induction argument.
\end{proof}
For $p>1$, Proposition~\ref{pr:limitationReLUNet} implies that ReLU NNs with $p$-norm constraints on the weights cannot reproduce the absolute value and a whole family of simple functions, including the triangular hat function (also known as the B-spline of degree 1) and the soft thresholding function.
Further, this result suggests that activation functions having more than one region with maximal slope are a better choice for this approximation framework.
Recall that learnable spline activation functions are capable of having this property.

\subsection{Limited Expressivity}
\label{subsec:variations}
A meaningful metric for the expressivity of a model is its ability to produce functions with high variation.
In this section, we investigate the impact of the Lipschitz constraint on the maximal second order total variation of such a NN.
Note that we partially rely on results in \cite{HCC2018} for our proofs.
The second order total variation of a function $f\colon \mathbb{R}\rightarrow\mathbb{R}$ is defined as
${\TV{2}}(f)\coloneqq \|\mathrm{D}^2 f\|_{\mathcal{M}}$, where $\|\cdot\|_{\mathcal{M}}$ is the total variation norm related to the space of bounded Radon measures $\mathcal M$, and $\mathrm{D}$ is the distributional derivative operator.
The space of functions with bounded second order total variation is denoted by
\[
\BV{2}(\mathbb{R}) = \{f\colon \mathbb{R}\rightarrow\mathbb{R} \colon \TV{2}(f)\leq +\infty\}.
\]
For more details, we refer the reader to \cite{Bredies_2020,unserRepresenterTheoremDeep2019}.
Further, we recall that $ \TV{2}$ is a semi-norm, which in case of a CPWL function on the real line is given by the finite sum of its absolute slope changes.
Based on  Lemma~\ref{lm:netTV2} below, we infer for the $p$-norm-constrained setting that, in general, a linear spline activation function cannot be replaced with a one layer ReLU NN without losing expressivity.

\begin{lemma}
\label{lm:netTV2}
Let  $f\colon \R \to \R$ be a function parametrized by a one hidden layer NN with component-wise activation function $\sigma$ and $p$-norm-constrained weights, $p\in[1,+\infty]$.
If $\sigma \in \BV{2}(\R)$, then
\begin{equation}
    \TV{2}(f)\leq \TV{2}(\sigma).
\end{equation}
\end{lemma}
\begin{proof}
    Let $f$ be given by $x\mapsto  u^T\sigma(wx + b) = \sum_{i=1}^N u_{i}\sigma(w_{i}x + b_{i})$ with $u \coloneqq (u_1,\ldots,u_N)\in \R^N$, $w \coloneqq (w_1,\ldots,w_N)\in \R^N$ and $b \coloneqq (b_1,\ldots,b_N)\in \R^N$.
    The $p$-norm weight constraints imply that $\|w\|_p\leq 1$ and $\|u\|_q\leq 1$ with $1/p+1/q=1$.
    Since $\TV{2}$ is a semi-norm, we get
    \begin{align}
        \TV{2}(f)&\leq \sum_{i=1}^N |u_{n}|\TV{2}(\sigma(w_{i}\cdot + b_{i}))\leq \sum_{i=1}^N |u_{i}w_{i}|\TV{2}(\sigma)\leq \TV{2}(\sigma),
    \end{align}
    where the last step follows by Hölder's inequality.
\end{proof}

In principal, the composition operation already suffices to increase the second order total variation of a mapping exponentially.
For instance, the $n$ fold composition $f_n$ of $f\colon \R \to \R$ with $x \mapsto 2\vert x - 1/2\vert$ yields the sawtooth function with $2^n$ linear regions and
\begin{equation}
    \mathrm{TV}^{(2)}(f_n) = 2(2^{n}-1).
\end{equation}
This highly desirable property is, however, not achievable by ReLU NNs with $\infty$-norm-constrained weights  \cite[Thm.~1]{HCC2018}.
As the next result shows, this has a drastic impact on the approximation power of ReLU NNs.
\begin{proposition}
Let $D\subset\mathbb{R}^d$ be compact.
Then, there exists $f \in \Lip_{1,\infty}(D)$ that cannot be approximated by NNs $\Phi\colon D\rightarrow\R$ with architecture~\eqref{eq:NN_archi}, $\infty$-norm-constrained weights and ReLU activation functions.
\end{proposition}
\begin{proof}
    By \cite[Thm.~1]{HCC2018}, we know that for any $u\in\R^d$ with $\Vert u\Vert_{\infty}=1$ and any ReLU NN $\Phi$ with $\infty$-norm weight constraints, it holds
    \[\TV{2}(\Phi\circ\varphi_u)\leq 2,\]
    where $\varphi_u\colon \R \rightarrow \R^d$ with $t \mapsto t u$.
    Let $(\Phi_n)_{n \in \N} \colon D\rightarrow \R$ be a uniformly convergent sequence of ReLU NNs with $\infty$-norm-constrained weights and limit $\Phi$.
    Then, $(\Phi_n \circ \varphi_u)_{n \in \N}$ converges uniformly to $\Phi \circ \varphi_u$ on $D$.
   Since $ \TV{2}$ is lower-semi continuous w.r.t.\ uniform convergence, see \cite[Prop.~3.14]{Bredies_2020}, we infer that
    \[\TV{2}(\Phi \circ\varphi_u)\leq 2.\]
    In other words, any $f \in \Lip_{1,\infty}(D)$  with $\TV{2} (f \circ \varphi_u) > 2$ for some $u \in \R^d$ cannot be approximated by $\infty$-norm constraint ReLU NNs.
    Note that the existence of such a function $f$ follows from Proposition~\ref{prop:DSVariations}.
\end{proof}

Unlike the ReLU activation function, linear spline activation functions can produce arbitrarily complex functions thanks to the composition operation, even in the norm-constrained setting.
\begin{proposition}
\label{prop:DSVariations}
Let $C>0$, $p\in[1,\infty]$ and $u\in\R^d$.
Then, there exists a NN $\Phi\colon \R^d\rightarrow \R$ with architecture~\eqref{eq:NN_archi}, $p$-norm-constrained weights and $1$-Lipschitz linear spline activation functions such that for $\varphi_u\colon \R \rightarrow \R^d$ with $t \mapsto t u$ it holds
\[
{\TV{2}} (\Phi\circ\varphi_u)>C.
\]
\end{proposition}
\begin{proof}
    Let $\sigma_k\colon x\mapsto |x|-1/2^k$ and $F_m=\sigma_m\circ\cdots\circ \sigma_1$.
    The function $F_m$ is a sawtooth-like CPWL function with $2^m$ linear regions.
    Further, it holds for all $t\in\R$ that $\vert F'_m(t) \vert = 1$, and the sign of the slope is different for neighboring regions.
   From this, we directly infer that 
    \begin{equation}
        \mathrm{TV}_2(F_m) = 2(2^m-1).
    \end{equation}
    Now, we build a deep spline NN $\Phi_K\colon \R^d \rightarrow \R$ with $K$ hidden layers of widths $n_1,\ldots,n_K =d$ and $n_{K+1}=1$.
    The activation function used in the $k$-th hidden layer is $\sigma_k$ for the first neuron and zero otherwise, the weight matrices are chosen as the identity matrix except for the last layer, where it is chosen such that
    \[
    \Phi_K(x) = F_K(x_1).
    \]
    This construction results in
    \begin{align}
    \TV{2}(\Phi_K\circ \varphi_{e_1}) &= 2(2^K-1),
      \end{align}
     and the claim follows for $u=e_1$ by taking $K$ sufficiently large.
     The general case $u\neq e_1$ follows by using an appropriate weight matrix in the first layer.
\end{proof}

\section{Approximating 1-Lipschitz Functions}\label{sec:ApproxLipNet}
In this section, we investigate the approximation of $1$-Lipschitz functions using the NN architecture~\eqref{eq:NN_archi} together with different activation functions and weight constraints.
Compared to the setting in Section \ref{sec:Universality}, the situation is a lot more involved.
\subsection{Networks with Component-wise Activation Functions}
Here, we investigate NNs with architecture as in \eqref{eq:NN_archi}, $p$-norm-constrained weights and with $1$-Lipschitz component-wise activation functions.
As a first step towards a better understanding, we restrict our attention to functions on the real line.
In the following, we show that any activation function $\sigma \colon \R \to \R$ can be written as composition of simple linear splines. 
\begin{proposition}\label{prop:1D_Approx}
Let $g \colon \R \to \R$ be a 1-Lipschitz CPWL function.
Then, there exist $n \in \N$ and 1-Lipschitz CPWL functions $g_i \colon \R \to \R$, $i=1,\ldots,n$, with at most 3 linear regions such that $g = g_n \circ \cdots \circ g_1$.
\end{proposition}
\begin{proof}
Note that we can restrict to functions $g$ with $\lim_{x \to \pm \infty}\vert \nabla g(x) \vert = 1$.
The general case can then be obtained by reparametrizing functions satisfying this condition such that they have the correct slope on the outmost linear regions.
We proceed via induction over the number of linear regions $m$ of $g$.
For $g$ with up to 3 linear regions the claim is clearly true.
Now assume that it is true for some $m \in \N$ and let $g$ be linear on the $m+1>3$ intervals $[a_i,a_{i+1}]$, $i = 0, \ldots, m$, with $a_0 = -\infty$ and $a_{m+1} = \infty$.
Now, we distinguish three cases.

\textbf{Case 1:} There exists some $a_j$, $j \in \{2,\ldots,m-1\}$, such that the function $g$ has a extremum in $a_j$ when restricted to $]-\infty, a_j]$ or $[a_j, \infty[$.
As all possible cases are similar, we only provide the construction for $a_j$ being a maximum on $]-\infty, a_j]$.
To this end, we define the function $\tilde g_1$, $\tilde g_2$ with
\begin{equation}
    \tilde g_1(x) =
    \begin{cases}
     g(x) \quad &\text{for } x \leq a_j,\\
     g(a_j) + (x-a_j) \quad &\text{for } x > a_j,
    \end{cases}
\end{equation}
and
\begin{equation}
    \tilde g_2(x) =
    \begin{cases}
     x \quad &\text{for } x \leq g(a_j),\\
     g\bigl(x + a_j-g(a_j)\bigr) \quad &\text{for } x > g(a_j),
    \end{cases}
\end{equation}
which are both 1-Lipschitz piecewise linear functions with at most $m$ linear regions satisfying $\lim_{x \to \pm \infty}\vert \nabla \tilde g_i(x) \vert = 1$.
Further, it holds $g = \tilde g_2\circ \tilde g_1$ and we can apply the induction assumption to conclude the argument.

\textbf{Case 2:} Case 1 does not apply and $\lim_{x \to \infty}\nabla g(x)/\nabla g(-x) = 1$.
In the following, we reduce this to Case 1.
We only provide the construction for $\lim_{x \to -\infty}\nabla g(x) = 1$, the other case is similar.
Here, it holds $g(a_1) \geq g(a_i) \geq g(a_m)$ for all $i=1,\ldots,m$ and we define functions $\tilde g_1$, $\tilde g_2$ with
\begin{equation}
    \tilde g_1(x) =
    \begin{cases}
     g(x) \quad &\text{for } x < a_1,\\
     2g(a_1) - g(x) \quad &\text{for } a_1 \leq x \leq a_m,\\
     g(x) + 2\bigl(g(a_1)-g(a_m)\bigr) \quad &\text{for } x > a_m,
    \end{cases}
\end{equation}
and
\begin{equation}
    \tilde g_2(x) =
    \begin{cases}
     x \quad &\text{for } x < g(a_1),\\
     2g(a_1) - x \quad &\text{for } g(a_1) \leq x \leq 2g(a_1) - g(a_m),\\
     2\bigl(g(a_m)-g(a_1)\bigr) + x \quad &\text{for } 2g(a_1) - g(a_m) < x.\\
    \end{cases}
\end{equation}
Clearly, both of the functions satisfy $\lim_{x \to \pm \infty}\vert \nabla \tilde g_i(x) \vert = 1$ and are 1-Lipschitz.
Here, the first function has $m+1$ linear regions and the second one has $3$.
Further, the first function now fits into Case 1 and it remains to show that $g = \tilde g_2 \circ \tilde g_1$.
However, this follows immediately from $g(a_1) \geq \tilde g_1(x) \geq g(a_1) - g(a_m)$ for $x \in [a_1,a_m]$.

\textbf{Case 3:} Case 1 does not apply and $\lim_{x \to \infty}\nabla g(x)/\nabla g(-x) = -1$.
In the following, we show that this case can be reduced to either Case 1 or Case 2.
We assume $\lim_{x \to -\infty}\nabla g(x) = 1$ and note that the other case is again similar.
Then, it holds $\min \{g(a_1),g(a_m)\} \geq g(a_i)$ for all $i=1,\ldots,m$ and we choose $a^* \in \argmax_{x \in \R} g(x) \in \{a_1,a_m\}$.
Next, we define functions $\tilde g_1$, $\tilde g_2$ via
\begin{equation}
    \tilde g_1(x) =
    \begin{cases}
     g(x) \quad &\text{for } x < a^*,\\
     2g(a^*) - g(x) \quad &\text{for } a^* \leq x,
    \end{cases}
\end{equation}
and
\begin{equation}
    \tilde g_2(x) =
    \begin{cases}
     x \quad &\text{for } x < g(a^*),\\
     2g(a^*) - x \quad &\text{for } g(a^*) \leq x.
    \end{cases}
\end{equation}
Note that both of the functions satisfy $\lim_{x \to \pm \infty}\vert \nabla \tilde g_i(x) \vert = 1$ and are 1-Lipschitz.
Here, the first function has $m+1$ linear regions and the second one has $2$.
Further, the first function now fits into either Case 1 or Case 2 and hence it remains to show that $g = \tilde g_2 \circ \tilde g_1$.
However, this follows immediately from the definition of $a^*$.
\end{proof}
\begin{remark}
The proof actually also shows that if $g$ satisfies $\vert \nabla g (x) \vert = 1$ a.e., then the same also holds true for the $g_i$.
Further, the result can be interpreted as approximation with a NN that has only one neuron per hidden layer.
Note that a similar approximation result for ResNets without Lipschitz constraints was given in \cite{LJ2018}.
\end{remark}
The previous result is a strong motivation for using deep spline NNs.
In particular, it implies that deep spline NNs with very simple activation functions already suffice to achieve the maximum representational power in \eqref{eq:NN_archi}.
\begin{theorem}\label{theorem:ds}
Let $D \subset \R^d$ be compact.
Then, NNs $\Psi\colon D \to \R^{n}$ with architecture~\eqref{eq:NN_archi}, $p$-norm-constrained weights, and $1$-Lipschitz spline activation functions with 3 linear regions can approximate the same functions as the corresponding NNs $\Phi\colon D \to \R^{n}$ with arbitrary $1$-Lipschitz component-wise activation functions.
\end{theorem}
\begin{proof}
We proceed by induction over the number $K$ of layers of $\Phi$.
For $K=1$, the NN $\Phi$ produces an affine mapping and there is nothing to show.
Assume that the statement holds for $K$ layers.
Let $\Phi_{K+1} \colon \R^d \to \R^{n_{K+1}}$ be a NN of the form~\eqref{eq:NN_archi} with $p$-norm-constrained weights and $K+1$ layers.
Then $\Phi_{K+1}= A_{K+1} \circ \sigma_{\alpha_K} \circ \Phi_K$ with a $K$ layer NN $\Phi_K\colon \R^d \to \R^{n_K}$ of the same form.
By applying the induction assumption, for any $\epsilon\in{\mathbb {R} _{>0}}$ there exists a deep spline NN $\Psi_1\colon \R^d \to \R^{n_K}$ with $p$-norm-constrained weights such that $\max_{x \in D} \Vert \Phi_K(x) - \Psi_1(x) \Vert_p \leq \epsilon/2$.
Due to the finite diameter of $D$, the range of $1$-Lipschitz functions is compact. 
Hence, Proposition~\ref{prop:1D_Approx} implies that there exists a deep spline NN $\Psi_2 \colon \R^{n_k} \to \R^{n_k}$ with all affine transformations being identities such that $\max_{x \in \Phi_K(D)} \Vert \sigma_{\alpha_K}(x) - \Psi_2(x) \Vert_p \leq \epsilon/2$.
For the deep spline NN $A_{K+1} \circ \Psi_2 \circ \Psi_1$ with spectral-norm-constrained weights, we can estimate
\begin{align}
    &\max_{x \in D} \Vert \Phi(x) - A_{K+1} \circ \Psi_2 \circ \Psi_1(x) \Vert_p \leq  \max_{x \in D} \Vert  \sigma_{\alpha_K} \circ \Phi_K(x) -  \Psi_2 \circ \Psi_1(x) \Vert_p\\
    \leq& \max_{x \in D} \Vert  \sigma_{\alpha_K} \circ \Phi_K(x) -  \Psi_2 \circ \Phi_K(x) \Vert_p + \Vert  \Psi_2 \circ \Phi_K(x) -  \Psi_2\circ \Psi_1(x) \Vert_p\\
    \leq&\epsilon/2 + \max_{x \in D} \Vert \Phi_K(x) -  \Psi_1(x) \Vert_p \leq  \epsilon.
\end{align}
This concludes the proof.
\end{proof}
Theorem \ref{theorem:ds} tells us that among all NNs of the form~\eqref{eq:NN_archi} with component-wise 1-Lipschitz activation functions, splines with 3 linear regions achieve the optimal representational power. However, resolving the question if NNs with $p$-norm-constrained weights are universal approximators for $\Lip_{1,p}(D)$ appears to be very challenging and is part of ongoing research.


\subsection{Groupsort vs Linear Spline Activation Functions}
In the following, we briefly discuss how Groupsort NNs and deep spline NNs can be expressed in terms of each other.
Here, the situation differs depending on the applied weight constraint.
First, we revisit a framework specifically tailored to Groupsort NNs, where the weights in the architecture~\eqref{eq:NN_archi} satisfy $\Vert W_k \Vert_\infty \leq 1$, $k=2,\ldots,K$, and $\Vert W_1 \Vert_{p,\infty} \leq 1$.
Then, expressing an arbitrary deep spline NN using a Groupsort NN is possible due to following universality result shown in \cite[Thm.~3]{ALG19}.
\begin{proposition}\label{prop:DenseGroupInf}
Let $D \subset \R^d$ be compact and $p \in [1,\infty]$.
The NNs with architecture~\eqref{eq:NN_archi}, Groupsort activation functions with groupsize at least 2, and with weight constraints as defined above are dense in $\Lip_{1,p}(D)$.
\end{proposition}
Although Proposition~\ref{prop:DenseGroupInf} states that density holds for all $p \in [1,\infty]$, this can be misleading as $p$ has only little to do with the involved norm constraints.
All but the first weight have to fulfill a $\infty$-norm constraint, which is rarely used in practice.
This somehow limits the practical relevance of the result.
Nevertheless, it would be interesting if a similar result also holds for deep spline NNs.
Let us remark that the proof of Proposition~\ref{prop:DenseGroupInf} relies heavily on the maximum operation and the chosen norms, which makes it difficult to generalize to other norm constraints or activation functions.

Now, we discuss the case of spectral norm constraints, which are the usual choice in practice.
For this setting, let us recall that it holds
\[\max(x_1,x_2) = \frac{x_1 + x_2 + \vert x_1 - x_2 \vert}{2}.\]
Hence, in case of spectral-constrained weights, the MaxMin activation function can be written as a deep spline NN, i.e., $\mathrm{MaxMin}(x) = W_2 \sigma_1 (W_1x)$ with
\[W_1 = W_2 = \frac{1}{\sqrt{2}}\begin{pmatrix} 1 & 1 \\ 1 &-1 \end{pmatrix} \quad \text{ and } \quad \sigma_1(x) = \begin{pmatrix} x_1 \\ \vert x_2 \vert \end{pmatrix}.\]
This can be extended to any Groupsort operation, since the MaxMin operation has the same expressivity as Groupsort under any $p$-norm constraint \cite{ALG19}.
For the reverse direction, i.e., rewriting a deep spline NN using a Groupsort NN with spectral-norm-constrained weights, we are not aware of any results.

\section{Conclusions and Open Problems}\label{sec:Conclusions}
In this paper, we have shown that NNs with linear spline activation functions with at least 3 linear regions attain the optimal approximation power for NNs with $p$-norm weight constraints and component-wise activation functions.
However, it remains an open question whether these NNs are universal approximators of $\Lip_{1,p}(D)$, $D\subset \R^d$ compact.
As this problem appears to be very challenging, our result could be a first step towards its solution.
The comparison of linear spline to non component-wise activation functions is subtle, and it so far unclear which choice leads to more expressive NNs.
For the spectral norm, we have shown that deep spline NNs are at least as expressive as Groupsort NNs, but for $\infty$-norm-constrained weights the opposite is true.
Further investigating the problem of universality under different constraints appears to be a promising research topic.
This possibly also leads to better trainable Lipschitz-constrained NN architectures.

Concerning the question of universality, we mainly focused on the approximation of scalar-valued functions $f \colon \R^d \to \R$.
This also reflects the current state of research, where most results are only formulated for scalar-valued NNs.
Extending these results to vector-valued functions appears highly non-trivial and should be addressed in future research.
Finally, we want to remark that little is known about the optimal structure for deep spline and Groupsort NNs, i.e., if it is more preferable to go deep or wide in architecture design.
\section*{Acknowledgment}
The research leading to these results was supported by the European Research Council (ERC) under European Union’s Horizon 2020 (H2020), Grant Agreement - Project No 101020573 FunLearn and by the Swiss National Science Foundation, Grant 200020 184646/1.
\bibliographystyle{abbrv}
\bibliography{references,references_additional}

\begin{thebibliography}{10}

\bibitem{ALG19}
C.~Anil, J.~Lucas, and R.~Grosse.
\newblock Sorting out {L}ipschitz function approximation.
\newblock In {\em Proceedings of the 36th International Conference on Machine
  Learning}, volume~97 of {\em Proceedings of Machine Learning Research}, pages
  291--301. PMLR, 2019.

\bibitem{pmlr-v70-arjovsky17a}
M.~Arjovsky, S.~Chintala, and L.~Bottou.
\newblock Wasserstein generative adversarial networks.
\newblock In {\em Proceedings of the 34th International Conference on Machine
  Learning}, volume~70 of {\em Proceedings of Machine Learning Research}, pages
  214--223. PMLR, 2017.

\bibitem{AU2019}
S.~Aziznejad and M.~Unser.
\newblock Deep spline networks with control of {L}ipschitz regularity.
\newblock In {\em IEEE International Conference on Acoustics, Speech and Signal
  Processing}, pages 3242--3246. IEEE, 2019.

\bibitem{BCGA2020}
P.~Bohra, J.~Campos, H.~Gupta, S.~Aziznejad, and M.~Unser.
\newblock Learning activation functions in deep (spline) neural networks.
\newblock {\em IEEE Open Journal of Signal Processing}, 1:295--309, 2020.

\bibitem{bohra2021learning}
P.~Bohra, D.~Perdios, A.~Goujon, S.~Emery, and M.~Unser.
\newblock Learning {L}ipschitz-controlled activation functions in neural
  networks for {P}lug-and-{P}lay image reconstruction methods.
\newblock In {\em NeurIPS 2021 Workshop on Deep Learning and Inverse Problems},
  2021.

\bibitem{Bredies_2020}
K.~Bredies and M.~Holler.
\newblock {Higher-order total variation approaches and generalisations}.
\newblock {\em Inverse Problems}, 36(12):123001, 2020.

\bibitem{BRRS2021}
L.~Bungert, R.~Raab, T.~Roith, L.~Schwinn, and D.~Tenbrinck.
\newblock {CLIP}: Cheap {L}ipschitz training of neural networks.
\newblock In {\em Scale Space and Variational Methods in Computer Vision},
  pages 307--319. Springer, Cham, 2021.

\bibitem{Calin2020}
O.~Calin.
\newblock {\em Deep Learning Architectures: A Mathematical Approach}.
\newblock Springer, Cham, 2020.

\bibitem{CN2016}
A.~Chernodub and D.~Nowicki.
\newblock Norm-preserving orthogonal permutation linear unit activation
  functions (oplu).
\newblock {\em arXiv:1604.02313}, 2016.

\bibitem{CHC2019}
J.~E. Cohen, T.~P. Huster, and R.~Cohen.
\newblock Universal {L}ipschitz approximation in bounded depth neural networks.
\newblock {\em arXiv:1904.04861}, 2019.

\bibitem{FRH2019}
M.~Fazlyab, A.~Robey, H.~Hassani, M.~Morari, and G.~Pappas.
\newblock Efficient and accurate estimation of {L}ipschitz constants for deep
  neural networks.
\newblock In {\em Advances in Neural Information Processing Systems 32}, pages
  11427--11438. Curran Associates, Inc., 2019.

\bibitem{GFPC18}
H.~Gouk, E.~Frank, B.~Pfahringer, and M.~Cree.
\newblock Regularisation of neural networks by enforcing {L}ipschitz
  continuity.
\newblock {\em Machine Learning}, 110:393--416, 2021.

\bibitem{GAAD2017}
I.~Gulrajani, F.~Ahmed, M.~Arjovsky, V.~Dumoulin, and A.~C. Courville.
\newblock Improved training of {W}asserstein gans.
\newblock {\em Advances in Neural Information Processing Systems 30}, pages
  2644--2655, 2017.

\bibitem{HN20}
P.~Hagemann and S.~Neumayer.
\newblock Stabilizing invertible neural networks using mixture models.
\newblock {\em Inverse Problems}, 37(8):085002, 2021.

\bibitem{HHNPSS2019}
M.~Hasannasab, J.~Hertrich, S.~Neumayer, G.~Plonka, S.~Setzer, and G.~Steidl.
\newblock Parseval proximal neural networks.
\newblock {\em The Journal of Fourier Analysis}, 26:59, 2020.

\bibitem{HNS2021}
J.~Hertrich, S.~Neumayer, and G.~Steidl.
\newblock Convolutional proximal neural networks and {P}lug-and-{P}lay
  algorithms.
\newblock {\em Linear Algebra and Applications}, 631:203--234, 2021.

\bibitem{huang2018orthogonal}
L.~Huang, X.~Liu, B.~Lang, A.~W. Yu, Y.~Wang, and B.~Li.
\newblock Orthogonal weight normalization: Solution to optimization over
  multiple dependent {S}tiefel manifolds in deep neural networks.
\newblock In {\em 32nd AAAI Conference on Artificial Intelligence}, pages
  3271--3278. AAAI Press, 2018.

\bibitem{HCC2018}
T.~Huster, C.-Y.~J. Chiang, and R.~Chadha.
\newblock Limitations of the {L}ipschitz constant as a defense against
  adversarial examples.
\newblock In {\em Joint European Conference on Machine Learning and Knowledge
  Discovery in Databases}, pages 16--29. Springer, Cham, 2018.

\bibitem{LRC2020}
F.~Latorre, P.~Rolland, and V.~Cevher.
\newblock Lipschitz constant estimation of neural networks via sparse
  polynomial optimization.
\newblock In {\em International Conference on Learning Representations}, 2020.

\bibitem{LHAL2019}
Q.~Li, S.~Haque, C.~Anil, J.~Lucas, R.~Grosse, and J.-H. Jacobsen.
\newblock Preventing gradient attenuation in {L}ipschitz constrained
  convolutional networks.
\newblock In {\em Advances in Neural Information Processing Systems 32}, pages
  15364--15376. Curran Associates, Inc., 2019.

\bibitem{LJ2018}
H.~Lin and S.~Jegelka.
\newblock Resnet with one-neuron hidden layers is a universal approximator.
\newblock In {\em Advances in Neural Information Processing Systems 31}, pages
  6172--6181. Curran Associates, Inc., 2018.

\bibitem{MMHC2017}
T.~Meinhardt, M.~Moeller, C.~Hazirbas, and D.~Cremers.
\newblock Learning proximal operators: {U}sing denoising networks for
  regularizing inverse imaging problems.
\newblock In {\em IEEE International Conference on Computer Vision}, pages
  1799--1808. IEEE, 2017.

\bibitem{MKKY2018}
T.~Miyato, T.~Kataoka, M.~Koyama, and Y.~Yoshida.
\newblock Spectral normalization for generative adversarial networks.
\newblock In {\em International Conference on Learning Representations}, 2018.

\bibitem{Pauli2022}
P.~Pauli, A.~Koch, J.~Berberich, P.~Kohler, and F.~Allg{\"{o}}wer.
\newblock Training robust neural networks using {L}ipschitz bounds.
\newblock {\em IEEE Control Systems Letters}, 6:121--126, 2022.

\bibitem{RKH2020}
K.~Roth, Y.~Kilcher, and T.~Hofmann.
\newblock Adversarial training is a form of data-dependent operator norm
  regularization.
\newblock In {\em Advances in Neural Information Processing Systems 33}, pages
  14973--14985. Curran Associates, Inc., 2020.

\bibitem{ryu2019plug}
E.~Ryu, J.~Liu, S.~Wang, X.~Chen, Z.~Wang, and W.~Yin.
\newblock {Plug-and-play methods provably converge with properly trained
  denoisers}.
\newblock In {\em International Conference on Machine Learning}, pages
  5546--5557. PMLR, 2019.

\bibitem{SGL2019}
H.~Sedghi, V.~Gupta, and P.~M. Long.
\newblock The singular values of convolutional layers.
\newblock In {\em International Conference on Learning Representations}, 2019.

\bibitem{SSF2021}
S.~Singla, S.~Singla, and S.~Feizi.
\newblock Householder activations for provable robustness against adversarial
  attacks.
\newblock {\em arXiv:2108.04062}, 2021.

\bibitem{SVW2016}
S.~Sreehariand, S.~V. Venkatakrishnan, and B.~Wohlberg.
\newblock Plug-and-play priors for bright field electron tomography and sparse
  interpolation.
\newblock {\em IEEE Transactions on Computational Imaging}, 2:408--423, 2016.

\bibitem{TSB2021}
U.~Tanielian, M.~Sangnier, and G.~Biau.
\newblock Approximating {L}ipschitz continuous functions with {G}roup{S}ort
  neural networks.
\newblock {\em arXiv:2006.05254}, 2021.

\bibitem{Tarela1990}
J.~M. Tarela, E.~Alonso, and M.~V. Mart{\'{i}}nez.
\newblock A representation method for {PWL} functions oriented to parallel
  processing.
\newblock {\em Mathematical and Computer Modelling}, 13(10):75--83, 1990.

\bibitem{TRPW20}
M.~{Terris}, A.~{Repetti}, J.~{Pesquet}, and Y.~{Wiaux}.
\newblock Building firmly nonexpansive convolutional neural networks.
\newblock In {\em IEEE International Conference on Acoustics, Speech and Signal
  Processing}, pages 8658--8662. IEEE, 2020.

\bibitem{TSS2018}
Y.~Tsuzuku, I.~Sato, and M.~Sugiyama.
\newblock Lipschitz-margin training: Scalable certification of perturbation
  invariance for deep neural networks.
\newblock In {\em Advances in Neural Information Processing Systems 31}, pages
  6542--6551. Curran Associates, Inc., 2018.

\bibitem{unserRepresenterTheoremDeep2019}
M.~Unser.
\newblock A representer theorem for deep neural networks.
\newblock {\em Journal of Machine Learning Research}, 20(110):1--30, 2019.

\bibitem{VBW13}
S.~V. Venkatakrishnan, C.~A. Bouman, and B.~Wohlberg.
\newblock Plug-and-play priors for model based reconstruction.
\newblock In {\em IEEE Global Conference on Signal and Information Processing},
  pages 945--948. IEEE, 2013.

\bibitem{VS2018}
A.~Virmaux and K.~Scaman.
\newblock Lipschitz regularity of deep neural networks: {A}nalysis and
  efficient estimation.
\newblock In {\em Advances in Neural Information Processing Systems 31}, pages
  3839--3848. Curran Associates, Inc., 2018.

\bibitem{pmlr-v139-zhao21e}
X.~Zhao, Z.~Zhang, Z.~Zhang, L.~Wu, J.~Jin, Y.~Zhou, R.~Jin, D.~Dou, and
  D.~Yan.
\newblock Expressive 1-{L}ipschitz neural networks for robust multiple graph
  learning against adversarial attacks.
\newblock In {\em Proceedings of the 38th International Conference on Machine
  Learning}, volume 139 of {\em Proceedings of Machine Learning Research},
  pages 12719--12735. PMLR, 2021.

\end{thebibliography}
\end{document}